 \documentclass[final,12pt,noproceedings]{clear2025} 

\usepackage{tikz}
\usetikzlibrary{arrows.meta, positioning}

\usepackage{ragged2e}
\usepackage{thm-restate}
\usepackage{subcaption}
\usepackage{centernot}

\newcommand{\kl}{D_{\text{KL}}}
\newcommand{\klv}{\hat{D}_{\scriptscriptstyle \text{KL}}^{\scriptscriptstyle \text{VM}}}

\usepackage{bbm}
\usepackage{amsfonts}
\usepackage{color}
\usepackage{mathtools}

\newtheorem{assumption}{Assumption}

\newcommand{\dR}{\mathbb{R}}
\newcommand{\dP}{\mathbb{P}}

\newcommand{\bx}{\mathbf{x}}
\newcommand{\by}{\mathbf{y}}

\newcommand{\bs}{\mathbf{s}}

\newcommand{\cX}{\mathcal{X}}

\newcommand{\cF}{\mathcal{F}}

\newcommand{\what}{\widehat}
\newcommand{\eps}{\varepsilon}

\title[Sample Complexity of Nonparametric Closeness Testing]{Sample Complexity of Nonparametric Closeness Testing for Continuous Distributions and Its Application to Causal Discovery with Hidden Confounding}

\usepackage{times}

\clearauthor{%
 \Name{Fateme Jamshidi} \Email{fateme.jamshidi@epfl.ch}\\
 \addr College of Management of Technology, EPFL, Lausanne, Switzerland
 \AND
 \Name{Sina Akbari} \Email{sina.akbari@epfl.ch}\\
 \addr Department of Computer and Communication Sciences, EPFL, Lausanne, Switzerland%
 \AND
 \Name{Negar Kiyavash} \Email{negar.kiyavash@epfl.ch}\\
 \addr College of Management of Technology, EPFL, Lausanne, Switzerland%
}

\begin{document}

\maketitle

\begin{abstract}
    We study the problem of closeness testing for continuous distributions and its implications for causal discovery.
    Specifically, we analyze the sample complexity of distinguishing whether two multidimensional continuous distributions are identical or differ by at least $\epsilon$ in terms of Kullback-Leibler (KL) divergence under non-parametric assumptions.
    To this end, we propose an estimator of KL divergence which is based on the von Mises expansion.
    Our closeness test attains optimal parametric rates under smoothness assumptions.
    Equipped with this test, which serves as a building block of our causal discovery algorithm to identify the causal structure between two multidimensional random variables, we establish sample complexity guarantees for our causal discovery method.
    To the best of our knowledge, this work is the first work that provides sample complexity guarantees for distinguishing cause and effect in multidimensional non-linear models with non-Gaussian continuous variables in the presence of unobserved confounding.
\end{abstract}

\section{Introduction}

The observation of a correlation between two variables $A$ and $B$ raises a fundamental question in causal inference: Does one variable cause the other, or is the correlation merely the result of a hidden confounding factor?
As depicted in Figure \ref{fig: 3cases}, the explanation generally falls into one of three possibilities: $A$ causes $B$, $B$ causes $A$, or a hidden variable $U$ causes both. 

Distinguishing among these causal structures is challenging.
It is well known that knowing the observational joint distribution $P(A, B)$ is not sufficient for this task, and it is necessary to perform \textit{interventions}.
In the framework of Pearl’s do-calculus \citep{pearl1995causal}, an intervention 
\( do(A = a) \) forces \( A \) to a specific value \( a \), allowing us to observe
changes in the distribution of \( B \). 
If the true structure is \( A \rightarrow B \), 
intervening on \( A \) will affect \( B \), but not vice versa.
For \( B \rightarrow A \), 
the reverse is true. In the case of \( A \leftarrow U \rightarrow B \), interventions on 
\( A \) or \( B \) will not influence the other, as their correlation is only due to \( U \).

A considerable body of work in the causal structure learning literature focuses on minimizing the number of required interventions for determining the causal structure \citep[e.g.,][]{eberhardt2007causation, eberhardt2012almost, hauser2012characterization, shanmugam2015learning, kocaoglu2017cost, lee2018structural, greenewald2019sample, squires2020active, mokhtarian2022learning}.
However, most do so assuming access to an infinite number of samples,
while in practice, the available data is often limited. 
Therefore, to ensure the reliability and applicability of causal discovery methods, it is necessary to establish sample complexity guarantees. 
A key question in this context is: How many samples are required to reliably infer the causal direction? 
In order to answer this question,
we will begin by analyzing the sample complexity of estimating a measure of closeness between certain distributions, for instance, to distinguish between the observed distribution of $B$ and the interventional distribution of $B$ under $do(A=a)$.
This will allow us to test whether there is indeed a causal influence from one variable to the other.
Existing methods for closeness testing offer theoretical guarantees in discrete or structured domains \citep{diakonikolas2015optimal, diakonikolas2021optimal, diakonikolas2024testing}.
However, extending these guarantees to continuous variables remains challenging due to the infinite support and the complexity of estimating densities. 
Our work fills this gap by developing a closeness testing framework for \textit{continuous} densities with theoretical guarantees. More specifically, we characterize the sample complexity of our test which allows us to apply it for causal discovery in the case of \emph{\bfseries continuous}, \emph{\bfseries multidimensional} variables. 

\begin{figure}[t]
    \centering
    \begin{tikzpicture}[node distance=0.1cm, thick, ->, >=Stealth]
        \node (X) at (0,0) [circle, draw, scale=0.8] {$A$};
        \node (Y) at (0,-2) [circle, draw, scale=0.8] {$B$};
        \draw[->] (X) -- (Y);
    \end{tikzpicture} \hspace{1cm}
    \begin{tikzpicture}[node distance=2cm, thick, ->, >=Stealth]
        \node (X) at (0,0) [circle, draw, scale=0.8] {$A$};
        \node (Y) at (0, -2) [circle, draw, scale=0.8] {$B$};
        \draw[->] (Y) -- (X);
    \end{tikzpicture} \hspace{1cm}
    \begin{tikzpicture}[node distance=2cm, thick, ->, >=Stealth]
        \node (X) at (0,0) [circle, draw, scale=0.8] {$A$};
        \node (Y) at (3,0) [circle, draw, scale=0.8] {$B$};
        \node (U) at (1.5,2) [circle, draw, dashed, scale=0.8] {$U$};
        \draw[->] (U) -- (X);
        \draw[->] (U) -- (Y);
    \end{tikzpicture}
    \caption{Three causal relationships between correlated variables $A$ and $B$.}\label{fig: 3cases}
\end{figure}
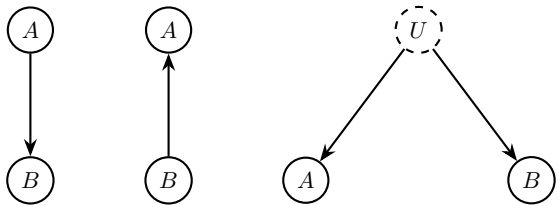

\noindent\textbf{Contributions}
Our main contributions are as follows:
\begin{itemize}
    \item [(i)]
    We derive an exponential concentration inequality for an estimator of KL divergence based on Von Mises expansion when the joint densities are estimated through kernel density estimation (KDE).
    \item [(ii)]
    Building on this result, we design a \emph{closeness test} to decide whether two continuous probability densities are equal or differ by at least a given amount $\epsilon$ in terms of their KL divergence.
    \item[(iii)] 
    We harness this test to identify the causal structure between two multidimensional random variables and provide sample complexity guarantees for our approach, in the presence of hidden confounding.
\end{itemize}

\paragraph{Outline of the paper}
In Section \ref{sec: background}, we briefly review the necessary background. 
In Section \ref{sec: concentration kl}, we establish the exponential concentration properties of a Von Mises estimator and propose a closeness test using this estimator for KL divergence. 
In Section \ref{sec: discovery},  we discuss how our closeness test can be harnessed to learn the causal relationship between $d$-dimensional continuous variables $A$ and $B$ and provide sample complexity guarantees for it. As a result of space constraints, some proofs are deferred to Appendix \ref{apx:proofs}.

\subsection{Related Work} 

%

Learning the causal structure between variables of a system has been the focus of intense research for decades \citep{chickering1997bayesian, spirtes2000causation, glymour2019review, scanagatta2019survey, gong2023causal}.
Given the difficulty of the problem, these approaches, for the most part, assume ideal conditions, such as access to an infinite number of samples, which is rarely met in real-world scenarios.
In recent years, attention has shifted toward the challenge of \textit{finite sample complexity} in causal discovery.
Several works have empirically studied the impact of sample size on the accuracy of causal inference \citep{eberhardt2010combining,mooij2016distinguishing, yang2018characterizing}.
From the theoretical standpoint, \citet{compton2022entropic} provided formal finite-sample guarantees for two-variable systems under assumptions of causal sufficiency and an assumption on the entropy of the exogenous variable.
\cite{wadhwa2021sample} studied the sample complexity of causal discovery for a network of \textit{discrete} variables by integrating finite-sample conditional independence tests, proposed in \citep{canonne2018testing} into the causal framework of \citep{pearl1995theory}.
This result does not extend to the \textit{continuous setting}.
Recent work by \citet{acharya2023sample} characterized the sample complexity required to distinguish cause from effect in bivariate one-dimensional \textit{discrete} settings, with a success probability of at least $2/3$. 
They showed that the necessary number of interventional samples depends on the domain size and characterized the trade-off between observational and interventional data.
Their work builds on the techniques developed by \citet{diakonikolas2021optimal},  who focused on testing the closeness of \textit{discrete} distributions with high probability, optimizing sample complexity as a function of parameters such as the error probability and domain size.
While these studies offer valuable insights, they do not apply to the continuous setting of our interest. Indeed, causal discovery in \textit{continuous settings} remains less explored, particularly in the finite-sample setting. 

\paragraph{KL Divergence and Von Mises Estimators}
Over recent decades, considerable attention has focused on estimating KL divergence.
Many of the existing methods are plug-in methods, i.e., they estimate the densities and evaluate the KL divergence functional based on these estimates.
\cite{singh2014exponential} 
among others established convergence rates for plug-in estimators for KL divergence.
In practice, evaluating the KL divergence numerically through its plug-in estimator becomes increasingly computationally expensive as the dimensionality of the variables grows. These estimators also suffer from slow convergence rates.
Another simple yet effective method for the estimation of KL divergence is the $k$-nearest neighbors (kNN) based method \citep[see e.g.,]{wang2009divergence, poczos2011estimation} although most of the work in this literature lacked convergence rate analyses.
In this context, \cite{noshad2017direct} provided convergence rate guarantees for a method of estimating Rényi and $f$-divergence measures via a graph-theoretic approach using kNN on joint data $(A, B)$.
However, both \citep{singh2014exponential} and \citep{noshad2017direct} have slower convergence rates and require stronger smoothness conditions than our approach in order to achieve similar convergence rates, a point we shall further discuss in the upcoming Section \ref{sec: concentration kl}.
More recently, \cite{zhao2020minimax} 
among others studied the sample complexity of kNN-based estimators and 
showed that they are asymptotically optimal under different assumptions than ours, such as weak tail distribution conditions.
Plug-in and kNN methods require undersmoothing the density estimate to achieve the best rate, and this smoothing parameter is in general unknown \citep{kandasamy2015nonparametric}. 

The Von Mises estimators have become a valuable tool for estimating statistical functionals, such as entropy,  mutual information, and divergence measures under nonparametric assumptions. 
These estimators, designed using the theory of influence functions and semi-parametric estimation \citep{fernholz2012mises, vd1998asymptotic}, are comprehensively studied for functionals of a single probability distribution (such as entropy).
\citet{kandasamy2015nonparametric} proposed and analyzed estimators for functionals of two densities (such as KL divergence) based on the Von Mises expansion.
This approach, previously applied in semiparametric settings \citep{birge1995estimation, robins2009quadratic}, corrects for the first-order bias terms in estimation, resulting in faster convergence rates.
Recently, building on this, \citet{jamshidi2023sample} applied the nonparametric Von Mises estimator to estimate mutual information specifically to conditional independence (CI) testing, the core component of constraint-based causal discovery algorithms.
Their work focused on recovering causal graphs up to the Markov equivalence class using only observational data.
However, their framework is restricted to independence testing based on mutual information estimates with access to joint densities. This limitation prevents its application in our setting, where we need to test the \emph{closeness} of two distributions using samples gathered under distinct conditions (e.g., either observational and interventional samples or two different interventions).


\paragraph{Sample complexity in closeness testing and causal discovery}

A line of work has focused on the sample complexity of closeness testing for discrete distributions to determine whether two distributions are identical or different.
\citet{diakonikolas2021optimal} studied the sample complexity of distinguishing if two discrete distributions (with a constant support size) are identical or their total variation distance is greater than $\epsilon$ with probability at least $1-\delta$ for given parameters $\epsilon,\delta$  and provided sample-optimal algorithms for this task.

In testing the closeness of two discrete distributions, sample complexity is largely dictated by the size of the domain.
Testing for continuous distributions poses a significant challenge due to infinite domain size.
In particular, if $p$ and $q$ are arbitrary continuous distributions, it is theoretically impossible to develop a finite-sample closeness tester with a constant probability of success \citep{batu2001testing}.
To address this challenge, two main approaches are commonly used in the literature.
The first involves imposing structural assumptions on $p$ and $q$. 
For instance, significant research has focused on closeness testing under Gaussianity and linearity assumptions \citep{diakonikolas2023gaussian, ingster2012nonparametric, verzelen2010goodness}.
The second approach is to approximate the closeness measure through techniques such as discretizing continuous densities to achieve a finite domain size. 
For instance, the $\mathcal{A}_k$ distance measures the maximum $l_1$-distance between the reduced distributions derived from $p$ and $q$ over all partitions of the domain into at most $k$ intervals.
\citet{diakonikolas2015optimal} proposed a sample-optimal algorithm for closeness testing within univariate distribution families based on $\mathcal{A}_k$ distance. 
Following this, \citet{diakonikolas2024testing} developed a closeness tester for multidimensional distributions, establishing upper and nearly-matching lower bounds for the sample complexity using tools from the Ramsey theory.
In this work, we will follow the first approach by imposing smoothness assumptions on the densities.

\section{Background}\label{sec: background}
We will use kernel density estimators (KDE) throughout this work.
For clarity, we include the formal definitions and properties of KDEs in Appendix \ref{apx:kernel}.
We refer the interested reader to \citep{terrell1992variable, chen2017tutorial} for further details.
Here, we begin by reviewing a well-known exponential concentration result for these estimators.
Building on this, we shall establish an exponential concentration inequality for a Von Mises estimator of KL divergence in the next section.

Before stating the concentration result, let us define a relevant notion, namely, the Hölder class, which is frequently used in the non-parametric estimation literature.
For a tuple \( \bs = (s_1, \ldots, s_d) \) of non-negative integers, we define \( |\bs| = \sum_{i=1}^{d} s_i \) and let the operator \( D^{\bs} \) denote
\(
D^{\bs} := \frac{\partial^{|\bs|}}{\partial x_1^{s_1} \cdots \partial x_d^{s_d}}.
\)
\begin{definition}[Hölder class]\label{def:holder}
    For $L>0$ and $\beta>0$, 
    the \emph{Hölder class} $\Sigma(\beta,L)$ on $\cX\subseteq\mathbb{R}^d$ is the set of functions $f:\cX\to\mathbb{R}$ that are $\lfloor\beta\rfloor $ times differentiable and satisfy
    \[
        |D^{\bs}f(x)-D^{\bs}f(y)|\leq L\|x-y\|^{\beta-|\bs|}
    \]
    for all $\bs = (s_1, \ldots, s_d)$ such that $|\bs|\leq\lfloor\beta\rfloor$.
    A function $f$ is called \emph{$\beta$-Hölder smooth} if $f \in \Sigma(\beta,L)$ for some $L>0$.
\end{definition}
Let $P$ be a probability measure on a compact space $\cX$ that is absolutely continuous with respect to the Lebesgue measure, and let $p$ denote its density function.
Let $K_d: \dR^d \to \dR$ be a kernel and let $\hat{p}_h$ be the corresponding kernel density estimator with bandwidth $h$ (see Eq.~\ref{eq:def_ph}.)
Under standard assumptions extensively studied in the literature (for instance, Assumption \ref{ass:kernel} in Appendix \ref{apx:kernel}), $\hat{p}_h$ is guaranteed to achieve minimax optimal rates of convergence to $p$.
Assumption \ref{ass:kernel} 
results in optimal convergence rates, specifically outlined below. 
The proof of this result follows from standard bias analysis and results of \citep{rinaldo2010generalized} -- see \citep{jamshidi2023sample} for instance.

\begin{restatable}[Exponential concentration of $\| p -\what{p}_h\|_{\infty}$]{proposition}{thmexpprob}\label{prp:conc_sup}
    Assume that $p$ belongs to the \emph{Hölder class} $\Sigma(\beta,L)$  on $\cX$ for some $\beta,L>0$ and that $K_d$ satisfies Assumption \ref{ass:kernel}.
   Let $h = h_n = \Theta(n^{-\frac{1}{2\beta + d}})$. Then, there exist $C_1, C_2, \eps_0>0$ and $n_0 \geq 0$ such that for all $ n^{-\frac{\beta}{2\beta + d}} (\log n)^{1/2} \leq \eps_n \leq \eps_0$:
    \begin{equation*}
       \forall n \geq n_0, \;  \dP\left(\|p-\what{p}_h\|_{\infty} > \eps_n \right) \leq C_1 \exp(- C_2 n^{\frac{2 \beta}{2 \beta + d}} \eps_n^2) \, .
    \end{equation*}
\end{restatable}

\subsection{Causal preliminaries}
We use structural causal models (SCMs) as the semantic framework of our work \citep{pearl2009causality}.
In this framework, causal relationships between variables can be described through deterministic functions and independent noise terms. For instance, for the three structures in~Figure~\ref{fig: 3cases}, the following hold. If the causal structure is $A \rightarrow B$, the variables are generated by the SCM given as:
$A := N_A, \quad B := f_B(A, N_B)$
where $N_A$ and $N_B$ are independent noise variables, and $f_B$ is a deterministic function.
The reverse structure, $B \rightarrow A$, can similarly be conceptualized as:
$B := N_B, \quad A := f_A(B, N_A)$.
In the presence of a hidden confounder $U$, causing both $A$ and $B$, the model can be described as:
$U := N_U, \quad A := f_A(U, N_A), \quad B := f_B(U, N_B)$
where $N_U$, $N_A$, and $N_B$ are independent noise terms, with $U$ as the latent variable inducing the correlation between $A$ and $B$. 

An \emph{intervention} on the variable $A$ refers to setting $A$ to a specific value $a$, overriding the natural value it would have taken.
This can be conceptualized as a modified SCM whereby $f_A(\cdot)$ is replaced by a constant function outputting the value $a$.
We will use Pearl's \emph{do} notation to represent interventions.
In particular, $P(B|do(A=a))$ represents the probability distribution induced over the variable $B$ in a modified SCM where the value of $A$ is fixed at $a$.
We will sometimes use the shorthand $P_a(B)$ when clear from context.
Analogously, $P_b(A)=P(A|do(B=b))$ represents the interventional distribution of $A$ under an intervention setting the value of $B$ to $b$.
\section{Exponential Concentration for KL Divergence Estimation}\label{sec: concentration kl}
Here, we present our first main result: the exponential concentration bound for a KL divergence estimator based on Von Mises expansion.
Let $P$ and $Q$ be two probability measures on a compact set $\cX\subseteq\dR^d$ that are absolutely continuous with respect to the Lebesgue measure, with continuous densities $p$ and $q$.
The KL divergence between $p$ and $q$ is defined as

\begin{equation}\label{eq: kl}
    \kl(p \| q) \coloneqq \mathbb{E}_P\big[\log \frac{p(x)}{q(x)}\big]=\int_{\cX} p(x) \log \frac{p(x)}{q(x)} \, dx,
\end{equation}
and  $\kl(p\|q)=0$ if and only if $p=q$ almost everywhere.
One common approach to estimating $\kl(p\|q)$ is to first estimate the distributions and then plug them into the above integral, formally:
\begin{equation*}\label{eq: plug}
    \hat{D}_{\scriptscriptstyle \text{KL}}^{ \text{plug-in}}
    (\hat{p}\|\hat{q}) 
    = \int_{\cX} \hat{p}(x) \log \frac{\hat{p}(x)}{\hat{q}(x)} \, dx,
\end{equation*}
where $\hat{p}$ and $\hat{q}$ denote the estimates of $p$ and $q$, respectively. 
This method, although intuitive, becomes computationally intractable 
especially in higher dimensions $d$.
Moreover, it results in slow convergence rates -- see Remark \ref{rem:rates}.
To address these issues, we will use an estimator based on the Von Mises expansion of KL divergence, which corrects for the first-order bias term, and, as we shall see, exhibits faster convergence rates.
Suppose we have $n$ and $m$ i.i.d. samples, $\{x_1^{i}\}_{1 \leq i \leq n} $ and $\{x_2^{j}\}_{1 \leq j \leq m} $ from $p$ and $q$, respectively.
We will use an estimation procedure based on data split as follows.
The dataset is divided into two subsets.
Density estimates $\hat{p}$ and $\hat{q}$ are constructed using the first subset $\{x_1^i\}_{i=1}^{n/2}$ and $\{x_2^j\}_{j=1}^{m/2}$, respectively.
These density estimates are plugged into the following Von Mises estimator using the second half of the data:
\begin{equation}\label{eq: vm}
        \klv(\hat{p}\|\hat{q}) =\Big(\frac{2}{n} \sum_{i=n/2+1}^n \log \frac{\hat{p}(x_1^i)}{\hat{q}(x_1^i)}\Big) + \Big(1-\frac{2}{m} \sum_{j=m/2+1}^m \frac{\hat{p}(x_2^j)}{\hat{q}(x_2^j)}\Big).
\end{equation}
Note that $\klv(p\|q)$ in Eq.~\eqref{eq: vm} can be read off as the sample analogue of 
\begin{equation}\label{eq: intvm}
   \int_\cX p(x)\log\frac{\hat{p}(x)}{\hat{q}(x)}dx+1-\int_\cX q(x)\frac{\hat{p}(x)}{\hat{q}(x)}dx ,
\end{equation}
which is equal to $\kl(p\|q)$ when $\hat{p}=p$ and $\hat{q}=q$.
In Appendix \ref{apx:vonmises}, we explain in detail the derivation of the estimator in Eq.~\eqref{eq: vm} based on the theory of influence functions \citep{newey1990semiparametric, ichimura2022influence}, and show that this estimator corrects for the first-order bias terms in the estimation of KL divergence.
In particular, using the Von Mises expansion of $\kl$ at $(p,q)$,
\begin{equation}
    \begin{aligned}
            \kl (p \| q) = & \kl(\hat{p} \| \hat{q} ) + \int \left( -\kl(\hat{p} \| \hat{q})+log\frac{\hat{p}(x)}{\hat{q} (x)} \right) p(x) \, dx + \int \left( 1- \frac{\hat{p}(x)}{\hat{q}(x)} \right ) q(x) \, dx \\
             &+ \mathcal{O} (\| p-\hat{p} \|_2^2) + \mathcal{O} (\| q -\hat{q} \|_2^2)\\
             =& \int p(x) \log \frac{\hat{p}(x)}{\hat{q}(x)} \, dx +1 - \int q(x) \frac{\hat{p}(x)}{\hat{q}(x)} \, dx  + \mathcal{O} (\| p-\hat{p} \|_2^2) + \mathcal{O} (\| q-\hat{q}\|_2^2),
    \end{aligned}
\end{equation}
which motivates the estimation of $\kl(p \| q)$ based on the form given in Eq. \eqref{eq: vm},
the sample analogue of Eq.~\eqref{eq: intvm}.
In what follows, we will use kernel density estimators $\hat{p}=\hat{p}_h$, and $\hat{q}=\hat{q}_h$.
To ensure that estimation is feasible, we require certain smoothness assumptions on the densities as follows.

\begin{assumption}[Assumptions on the densities $p$ and $q$.]\label{ass:density}
Densities $p$ and $q$ belong to the \emph{Hölder classes} $\Sigma(\beta_p,L)$ and $\Sigma(\beta_q,L)$, respectively, for some $L>0$.
Moreover, $p$ and $q$ are lower-bounded on $\cX$ by some $p_{\min} >0$ and $q_{\min} >0$.


\end{assumption}


\begin{remark}
    Next theorem shows that assuming both \( p \) and \( q \) are bounded below by some positive constant \( p_{\min} \) is essential to ensuring that the estimator \(\klv\) achieves an exponential convergence rate.
    When dealing with densities that do not meet this lower bound naturally, a practical solution is to truncate \( p \) and \( q \) on a sufficiently large compact interval, keeping the KL divergence close to its true value.
\end{remark}

\begin{restatable}[Exponential concentration of $\klv$ in Eq.~\ref{eq: vm}]{theorem}{thmexpkl}\label{th: con_kl}
Suppose that $K_d$ satisfies Assumption \ref{ass:kernel} and that densities $p$ and $q$ both meet Assumption \ref{ass:density}.
Let the bandwidth of the kernel estimates for densities $p$ and $q$ equal $h_p(n)= \Theta(n^{-\frac{1}{2\beta_p+d}})$ and $h_q(m)= \Theta(m^{-\frac{1}{2\beta_q+d}})$, respectively.
Then there exist positive constants $n_0,m_0$, $\{C_i, C'_i\}_{1 \leq i \leq 4}$ and $\epsilon_0>0$ such that for any $n>n_0$, $m>m_0$, and\\
\(
\max\{
n^{-2\beta_p/(2\beta_p+d)}\log n, n^{-1/2}, m^{-2\beta_q/(2\beta_q+d)}\log m, m^{-1/2}
\}
\leq \epsilon \leq \epsilon_0
\)
such that

\begin{equation}
    \begin{aligned}
         Pr \Big( \left|  \klv(\hat{p}_{h_p}\|\hat{q}_{h_q}) - \kl(p \| q)\right | &> \epsilon \Big) \leq   C_1 \exp \Big( -C'_1 n^{1/2} \epsilon \Big) + C_2 \exp \Big( -C'_2 m^{1/2} \epsilon \Big)\\
         & + C_3 \exp \Big( -C'_3 n^{\frac{2\beta_p}{2 \beta_p +d}} \epsilon \Big) + C_4 \exp \Big( -C'_4 m^{\frac{2\beta_q}{2 \beta_q +d}} \epsilon \Big).
    \end{aligned}
\end{equation}
\end{restatable}
The proof of Theorem \ref{th: con_kl} is given in Appendix \ref{apx:proofs}.
\begin{remark}
  When the smoothness parameters satisfy $\beta_p,\beta_q > \frac{d}{2}$, the estimator $\klv$ achieves the optimal parametric convergence rate of $\mathcal{O}(n^{-1/2}+m^{-1/2})$ which is the best possible rate \citep{birge1995estimation, laurent1996efficient}.
\end{remark}
\begin{remark}
    In the remainder of the paper, we assume $\beta_p=\beta_q=\beta$ to simplify the presentation.
\end{remark}
\begin{remark}\label{rem:rates}
    From Theorem \ref{th: con_kl}, we directly obtain the well-established convergence rate $\mathcal{O}(n^{-\lambda}+m^{-\lambda})$ where $\lambda = \min\{\frac{1}{2}, \frac{2\beta}{2\beta+d}\}$ for the Von Mises KL divergence estimator \citep[see][]{kandasamy2015nonparametric}. 
    In comparison, \citet{singh2014exponential} and \citet{noshad2017direct} obtain the slower rate of $\mathcal{O}(n^{-\min \{{\frac{\beta}{\beta + d}}, 1/2 \}})$ for $n$ samples from $p, q$.
\end{remark}

\subsection{Closeness Testing of Two Continuous Distributions}\label{sec:hypothesis}
In this section, we use Theorem \ref{th: con_kl} to devise a test to distinguish between the following two hypotheses 
\begin{equation*}
    H_0 := \kl(p \| q)=0 \mbox{ vs. } H_1 := \kl(p \| q) > \epsilon\, .
\end{equation*}
We will use the Von Mises estimator of Eq.~\eqref{eq: vm} to estimate $\kl(p\|q)$ given $n$ and $m$ samples from $p$ and $q$, respectively, and use the following test criterion:
\begin{equation}\label{eq:cor:kl}
        CT_{\mathrm{VM}}(n,m,\epsilon) := \begin{cases}
            H_1 \quad \mbox{if } \klv(\hat{p}\|\hat{q}) > \epsilon/2, \\
            H_0 \quad \mbox{otherwise. }
        \end{cases}
    \end{equation}
The following corollary is straightforward to verify using Theorem \ref{th: con_kl}.
\begin{corollary}[Sample Complexity of Closeness Testing]
    Under the conditions in Theorem \ref{th: con_kl}, given $n$ and $m$ i.i.d. samples drawn from $p$ and $q$, respectively, $CT_{\mathrm{VM}}(n,m,\epsilon)$ distinguishes correctly between $p=q$ and $\kl( p \| q)> \epsilon$ with probability at least $1-\delta$ where $\epsilon$ and $\delta$ are positive constants as long as:
    $$n,m = \Omega \left( (\frac{1}{\epsilon} \log \frac{1}{\delta})^\tau \right), $$ 
    where $\tau = \max\{2, \frac{2 \beta+d}{2 \beta}\}$.
\end{corollary}

\section{Determining Causal Relationships}\label{sec: discovery}
We now turn our focus to the problem of identifying the causal structure between two correlated continuous $d$-dimensional random variables $A$ and $B$ using both observational and interventional data.
In Subsection \ref{subsec:interv}, we consider the case where only interventional data is accessible.
We formalize the correlation between $A$ and $B$ as follows.
\begin{assumption}\label{as:cor}
    $A$ and $B$ are correlated. Formally,
    $\kl \left(P(A, B) \| P(A) P(B) \right) > \epsilon$ for some $\epsilon>0$.
\end{assumption}
Our goal is to determine whether the relationship between $A$ and $B$ is that of direct causation, that is $A \to B$ or $B \to A$, or an unobserved confounder $U$ influences both variables, resulting in the structure $A \gets U \to B$ (See Figure \ref{fig: 3cases}). 
We will determine the causal structure
by examining changes in the distributions under interventions.
We explain this below.

Suppose the true causal relationship is $A\to B$.
In this case, the interventional distribution $P_a(B) :=P(B|do(A=a))$ matches $P(B|A=a)$ for every $a$.
To test for the latter, we can use interventional samples drawn from $P_a(B)$ and observational samples drawn from $P(A, B)$ to estimate the densities and use the hypothesis testing procedure outlined in Section \ref{sec:hypothesis}.
On the other hand, if there is no causal edge from $A$ to $B$, i.e., $A\not\to B$, then $P_a(B)$ coincides with $P(B)$ for every $a$, which can be tested similarly.
In case our analysis does not imply $A\to B$, we can analogously test for $B\to A$ by estimating the KL divergence between $P_b(A)$ and $P(A|B=b)$.
If neither direction is implied by these tests, we conclude that the relationship is likely $A \gets U \to B$, where $U$ is an unobserved variable causing both $A$ and $B$.
To present our formal study, we will make use of two preliminary lemmas:

\begin{restatable}{lemma}{lemkl}\label{lem:expkl}
    Under Assumption \ref{as:cor}, 
    \[
        \mathbb{E}_{A} \left[\kl \left(P (B|A) \| P(B) \right) \right] > \epsilon \quad \text{and} \quad \mathbb{E}_{B} \left[\kl \left(P(A|B) \| P(A) \right) \right] > \epsilon.
    \]
\end{restatable}
See Appendix \ref{apx:proofs} for the proof.
\begin{lemma}[\citealp{levin1985one}; see also Fact A.2 in \citealp{goldreich2014multiple}] \label{lemma: levin}
    Let $P$ be a probability measure, and let $h: \text{supp}(P) \to [0,1]$ be a function with $\mathbb{E} [h(t)] > \epsilon$ for some $\epsilon \in (0,1]$.
    Define $k = \lceil \log_2 \frac{2}{\epsilon} \rceil$, $\epsilon_j= 2^{-j}$, and $r_j=\frac{2^{j} \epsilon}{(k+5-j)^2}$.
    Then, there exists $j \in [k]$ such that $\Pr( h(t) > \epsilon_j) > r_j$.
\end{lemma}
We begin by outlining the procedure for testing whether the edge $A\to B$ exists below. 
\subsection{Testing for \texorpdfstring{$A\to B$}{A to B} Using Observational and Interventional Data}

If $A \not\to B$, then $A$ has no causal effect on $B$, and $P_a(B)=P(B)$ for every $a$.
Clearly, \[\mathbb{E}_{A} \left[\kl \left(P(B|do(A)) \| P(B) \right) \right]=0\] in this case.
Conversely if $A\to B$, i.e., $A$ causally affects $B$, then $P_a(B)=P(B|A=a)$ for every $a$ and hence from Lemma \ref{lem:expkl},
$\mathbb{E}_{A} \left[\kl \left(P (B|do(A)) \| P(B) \right) \right]>\epsilon$.
This brings us to the following criterion for testing the edge $A\to B$:
\begin{equation}\label{eq:test}
    \begin{cases}
        A\to B\quad\Leftrightarrow\quad \mathbb{E} \left[h(A)\right]>\epsilon, \\
        A\not\to B\quad\Leftrightarrow\quad\mathbb{E} \left[h(A) \right]=0,
    \end{cases}
\end{equation}
where
\begin{equation}\label{eq:h}h(a)\coloneqq\kl\left(P(B|do(A=a)\|P(B)\right).\end{equation}
In the sequel, we show how the criterion of Eq.~\eqref{eq:test} can be verified with high probability, using the hypothesis testing procedure of Section \ref{sec:hypothesis} and Lemma \ref{lemma: levin}.
Before moving forward, we note that Lemma \ref{lemma: levin} requires $h(\cdot)$ to take values in $[0,1]$. 
KL divergence is non-negative, but can be unbounded in general.
However, under Assumption \ref{ass:density}, KL divergence remains upper-bounded due to $p$ and $q$ being bounded away from zero.
Hence, Lemma \ref{lemma: levin} remains valid with an appropriate scaling of $h$ based on the constants $p_{min}$ and $q_{min}$, which does not affect our asymptotic analysis.

Let $k,\epsilon_j$, and $r_j$ be defined as in Lemma \ref{lemma: levin}.
If the edge $A\to B$ exists, then Lemma \ref{lemma: levin}
implies that there exists an index $j^* \in [k]$ for which:
\[
\Pr \left(\kl \left(P (B|a) \| P(B) \right) > \epsilon_{j^*} \right) > r_{j^*},
\]
since $\mathbb E[h(A)] > \epsilon$.
This implies the following lemma.
\begin{restatable}{lemma}{lemconstant}\label{lem:constantpr}
    Suppose the causal structure is $A\to B$, and let $k, \epsilon_j$, and $r_j$ be defined as in Lemma \ref{lemma: levin}.
    There exists an index $j^*\in[k]$ such that
    for any $c>0$ and 
    given $l_{j^*}=\frac{2c+2}{r_{j^*}}$ i.i.d. samples $\{a_i\}_{i=1}^{l_{j^*}}$ from $P(A)$, at least one of these samples satisfies $h(a_i)>r_{j^*}$ with probability $1 - e^{-c}$.
\end{restatable}
The proof of this lemma, which is included in Appendix \ref{apx:proofs} for the sake of completeness, goes through an application of the Chernoff inequality.
Note that in the alternative case, i.e., $A\not\to B$, $h(a_i)$ is always zero.
Therefore, our algorithm tests whether there is an edge $A\to B$ as follows:
for each $j\in[k]$, we draw $l_j=\frac{c+2}{r_j}$ i.i.d. samples from $A$; for each sample $a_i$, we draw $n_j$ and $m_j$ samples from $p=P \left(B\do(A=a_i) \right)$ and $q=P(B)$, respectively;
we run the hypothesis test of Section \ref{sec:hypothesis} to test whether $h(a_i)>\epsilon_j$ or $h(a_i)=0$;
finally, if any of these tests return the hypothesis $H_1$ (i.e., $h(a_i)>\epsilon_j$) then we conclude the edge $A\to B$ exists; otherwise we conclude $A\not\to B$.
The pseudocode for this algorithm is presented as Algorithm \ref{alg:test}.

\begin{algorithm2e}[t]
\caption{Algorithm for Testing the Edge $A\to B$ Using Observational Data}
\label{alg:test}
\KwIn{ parameter $c$, smoothness $\beta$, threshold $\epsilon$, sample access to distributions $P(B)$ \& $P_A(B)$}
\KwOut{Determine whether the causal structure is $A \to B$}

Set params $\tau=\max\{2,\frac{2\beta+d}{2\beta}\}$, $k = \log \left( \frac{2}{\epsilon} \right)$, $\epsilon_j= 2^{-j}$, $r_j=\frac{2^j\epsilon}{(k+5-j)^2}$, $\delta_j = \frac{2 ^{j-k}}{(2c+2) (k+5-j)^4}$\;
Define sample sizes: $n_j = m_j = (\frac{1}{\epsilon_j}\log\frac{1}{\delta_j})^\tau$\;
\For{$j = 1$ \KwTo $k$}{
    \For{$i = 1$ \KwTo $\left\lfloor \frac{2c+2}{r_j} \right\rfloor$}{
        Draw sample $a_i \sim P(A)$\;
        \If{$CT_{VM}\left( n_j, m_j, \epsilon_j\right) = H_1$ for $p=P_{a_i}(B)$ and $q=P(B)$}{
            \Return $A \to B$
        }
    }
}
\Return $A \not\to B$\;
\end{algorithm2e}
The following result presents the sample complexity of our method.
\begin{restatable}{theorem}{thmsample}\label{thm:sample}
    Suppose Assumption \ref{as:cor} holds, kernel $K_d$ satisfies Assumption \ref{ass:kernel}, and that observational and interventional densities satisfy Assumption \ref{ass:density}.
    For any $c>0$, Algorithm \ref{alg:test} correctly distinguishes between $A\to B$ and $A\not\to B$ with probability $0.9(1-e^{-c})$, using $\mathcal{O}\big(\frac{c}{\epsilon^\tau}\big)$ observational and $\mathcal{O}\big(\frac{c}{\epsilon^\tau}\big)$ interventional samples, where $\tau=\max\{2,(2\beta +d)/2\beta\}$.
\end{restatable}
The number of interventional samples used throughout the algorithm is
\begin{equation}
    \begin{aligned}
            \frac{\epsilon^\tau}{\epsilon^\tau}\sum_{j \in [k]} \frac{2c+2}{r_j} n_j &=  \frac{2c+2}{\epsilon^\tau} \sum_{j = 1}^k \left (\frac{\epsilon}{\epsilon_j} \right)^{\tau-1} (k+5-j)^2 \log^\tau  \frac{(2c+2) (k+5-j)^4}{2^{j-k}} \\
            &= \frac{20}{\epsilon^\tau} \sum_{j =1}^k 2^{-j(\tau-1)} (j+5)^2 \log^\tau  \left ((2c+2) (j+5)^4 2^{j} \right)\\
            &= \frac{20\times2^{5(\tau-1)}}{\epsilon^\tau} \sum_{j =6}^k  \frac{j^2}{(2^{\tau-1})^j} \log^\tau  \left ((2c+2) j^4 2^{j-5} \right)\\
            & = \frac{1}{\epsilon^\tau}\times\mathcal{O}(c)=\mathcal{O}(\frac{c}{\epsilon^\tau}).
    \end{aligned}
\end{equation}
Similarly, the number of observational samples of $P(B)$ is the same.
Finally, the number of observational samples from $P(A)$ is $\sum_{j\in[k]}\frac{2c+2}{r_j}$, which is fewer than that required for $P(B)$ and can be bounded in a similar fashion.


\paragraph{Error analysis}
To evaluate the probability of error, note that the total number of closeness tests performed is bounded by $\sum_{j \in [k]} \frac{2c+2}{r_j} $, where each has an error probability of at most $\delta_j=\frac{2^{j-k}}{(2c+2)(k+5-j)^4}$.
Using union bound, the probability of event $E$ that at least one test results in an error is at most
\begin{equation}
    \begin{aligned}
            Pr( E) &\leq \sum_{j \in [k]} \frac{2c+2}{r_j} \cdot \delta_j \\
            &=
            \sum_{j \in [k]} \frac{(2c+2) (k+5-j)^2}{2^j \epsilon} \cdot  \frac {2^{j-k}}{(2c+2) (k+5-j)^4 } 
            \\ &=  \frac{1}{2} \sum_{j \in [k]} \frac{1}{(k+5-j)^2} = \frac{1}{2} \sum_{j =5}^{k+6} \frac{1}{j^2}
            \\&<\frac{1}{2}\int_5^\infty\frac{1}{x^2}dx= 0.1,
    \end{aligned}
\end{equation}
which implies that \emph{all} tests return correct answers with probability at least  $0.9$.
Finally, from Lemma \ref{lem:constantpr}, if the true structure is $A\to B$, then at least one test will return $H_1$ with probability $1-e^{-c}$, and therefore the output of the algorithm is correct with probability at least $0.9(1-e^{-c})$.
Conversely, if $A\not\to B$, all tests return $H_0$ and the algorithm returns the correct output with probability $0.9$.

\subsection{Complete Causal Discovery Algorithm}
We presented the procedure for deciding between $A\to B$ and $A\not\to B$.
The same procedure can be used to decide between $A\gets B$ and $A\not\gets B$.
Finally, if we conclude that none of the edges $A\to B$ and $A\gets B$ exist, the correlation between $A$ and $B$ is explained through the hidden confounder $U$: $A\gets U\to B$.
Based on Theorem \ref{thm:sample} and a symmetrical result for $B\to A$ versus $B\not\to A$,
along with an application of union bound we obtain the following result.
\begin{corollary}
    Under Assumption \ref{ass:density} for observational and interventional densities, Assumption \ref{as:cor}, and Assumption \ref{ass:kernel} for the kernel, for any $c>0$, the causal structure among $\{A\to B, A\gets B, A\gets U\to B\}$ can be correctly identified with probability at least $0.8(1-2.25e^{-c})$ given $\mathcal{O}(\epsilon^{-\tau})$ observational and interventional samples, where $\tau=\max\{2,(2\beta+d)/2\beta\}$.
\end{corollary}

\begin{remark}
    Note that although this causal discovery method is initially presented for a constant probability of success, it can be boosted to achieve a success rate of $1-\delta$ for any $\delta>0$ by employing the median trick \citep{jerrum1986random}.
    This approach involves enumerating the possible causal structures and running our causal discovery method $ \log(1/\delta)$ times and returning the median result, achieving an arbitrarily high probability of success with only a logarithmic factor increase in sample complexity.
\end{remark}

\subsection{Testing for \texorpdfstring{$A\to B$}{A to B} Without Observational Samples}\label{subsec:interv}
In certain applications such as online learning, observational samples might not be available.
Herein, we analyze how the causal structure can be identified using only interventional data.
Analogous to the previous section, we begin by presenting a method to decide between $A\to B$ and $A\not\to B$, which can be combined with
its counterpart for the reverse direction ($B\to A$ versus $B\not\to A$) 
to form a complete causal discovery algorithm.
We will use the following lemma, which is the analogue of Lemma \ref{lem:expkl}.
The proof is deferred to Appendix \ref{apx:proofs}.
\begin{algorithm2e}[t]
\caption{Algorithm for Testing the Edge $A\to B$ Without Observational Data}
\label{alg:test2}
\KwIn{ parameter $c$, smoothness $\beta$, threshold $\epsilon$, sample access to distributions $P_B(A)$ \& $P_A(B)$}
\KwOut{Determine whether the causal structure is $A \to B$}

Set params $\tau=\max\{2,\frac{2\beta+d}{2\beta}\}$, $k = \log \left( \frac{2}{\epsilon} \right)$, $\epsilon_j= 2^{-j}$, $r_j=\frac{2^j\epsilon}{(k+5-j)^2}$, $\delta_j = \frac{2 ^{j-k}}{(2c+2) (k+5-j)^4}$\;
Define sample sizes: $n_j = m_j = (\frac{1}{\epsilon_j}\log\frac{1}{\delta_j})^\tau$\;
\For{$j = 1$ \KwTo $k$}{
    \For{$i = 1$ \KwTo $\left\lfloor \frac{2c+2}{r_j} \right\rfloor$}{
        Draw i.i.d. samples $a_i,\tilde{a}_i \sim P_b(A)$\;
        \If{$CT_{VM}\left( n_j, m_j, \epsilon_j\right) = H_1$ for $p=P_{a_i}(B)$ and $q=P_{\tilde{a}_i}(B)$}{
            \Return $A \to B$
        }
    }
}
\Return $A \not\to B$\;
\end{algorithm2e}
\begin{restatable}{lemma}{leminterv}\label{lem:interv}
    Under Assumption \ref{as:cor},
    \[\begin{split}
        &\mathbb{E}_{(A,\tilde{A})\sim P(A)\times P(A)} \left[\kl \left(P (B|A) \| P(B|\tilde{A}) \right) \right] > \epsilon,\quad \text{and} \\& \mathbb{E}_{(B,\tilde{B})\sim P(B)\times P(B)} \left[\kl \left(P(A|B) \| P(A|\tilde{B}) \right) \right] > \epsilon.  
    \end{split}
    \]    
\end{restatable}
As discussed before, if the true structure is $A\to B$, then $P(B|do(A=a))= P(B|A=a)$ and therefore $\mathbb{E}_{(A,\tilde{A})\sim P(A)\times P(A)} \left[\kl \left(P (B|do(A)) \| P(B|do(\tilde{A})) \right) \right] > \epsilon$.
Otherwise, i.e., if $A\not\to B$, this expectation is $0$ since simply $P(B|do(A))=P(B)$.
In this setting, we have the following criterion for testing the existence of the edge $A\to B$ (analogue of Eq.~\ref{eq:test}):
\begin{equation}\label{eq:test2}
    \begin{cases}
        A\to B\quad\Leftrightarrow\quad \mathbb{E} \big[h_2(A,
        \tilde{A})\big]>\epsilon, \\
        A\not\to B\quad\Leftrightarrow\quad\mathbb{E} \big[h_2(A,\tilde{A}) \big]=0,
    \end{cases}
\end{equation}
where
\begin{equation}\label{eq:h2}h_2(a,\tilde{a})\coloneqq\kl\big(P(B|do(A=a)\|P(B|do(A=\tilde{a}))\big).\end{equation}
Based on this criterion, we adopt a similar testing method, outlined as Algorithm~\ref{alg:test2}.
The workings of this algorithm is similar to Algorithm~\ref{alg:test}.
The following result presents the sample complexity of this algorithm.
\begin{restatable}{theorem}{thmsample2}\label{thm:sampleinterv}
    Suppose Assumption \ref{as:cor} holds, the kernel $K_d$ satisfies Assumption \ref{ass:kernel}, and that interventional densities satisfy Assumption \ref{ass:density}.
    For any $c>0$, Algorithm \ref{alg:test2} correctly distinguishes between $A\to B$ and $A\not\to B$ with probability $0.9(1-e^{-c})$, using $\mathcal{O}\big(\frac{c}{\epsilon^\tau}\big)$ interventional samples, where $\tau=\max\{2,(2\beta +d)/2\beta\}$.
\end{restatable}
The proof is identical to Theorem \ref{thm:sample}.
Note that as before, the median trick can be applied to achieve arbitrarily low error probabilities.
Furthermore, the edge $B\to A$ can be tested in a similar fashion, and if neither $A\to B$ nor $B\to A$ are confirmed, we conclude that $A\gets U\to B$ is the true causal structure.
\section{Conclusion}
We presented an exponential concentration bound for a first-order Von Mises estimator of KL divergence.
We developed a hypothesis testing procedure based on this estimator and analyzed its sample complexity.
We then studied the problem of causal discovery involving multidimensional continuous variables in the presence of hidden confounding and developed an algorithm that correctly identifies the true causal structure with a constant probability and analyzed its sample complexity.
Boosting approaches can be applied to achieve higher success probabilities at the cost of logarithmic factor increase in sample complexity.
\section*{Acknowledgments}
This research was in part supported by the Swiss National Science Foundation under NCCR Automation, grant agreement 51NF40\_180545.
\bibliography{biblio}

\appendix
\clearpage

\section{Multivariate Kernel Density Estimation}
\label{apx:kernel}

Multivariate kernel density estimation (KDE) provides an approximation of the density $p$ given by

of the following form. For all 
\begin{equation}\label{eq:def_ph}
    \what{p}_h(x) := \frac{2}{n} \sum_{i=1}^{n/2} \frac{1}{h^d} K_d \left( \frac{x^{i} - x}{h}\right) , 
\end{equation} 
where $x = (x_1, \ldots, x_d)$ in $\cX$.
Here, $h := h(n)>0$ is the \emph{bandwidth} and $K_d: \dR^d \to \dR$ is a \emph{kernel} function with $\int K_d(x) \, dx = 1$, ensuring $\int_{\cX} \what{p}_h(x) \, dx =1$. 
Recall that for this estimation we only use the first half of the samples, i.e., $(x^{i})_{1 \leq m \leq n/2}$.

While the selection of $K_d$ remains flexible, higher-order kernels (order $\ell > 0$) are particularly effective for approximating smooth densities. We provide their formal definition below.

\begin{definition}[Kernels of given order]\label{def:kernel_order}
    Let $\ell$ be a positive integer. We say that a kernel $K_d: \dR^d \to \dR$ is a \emph{kernel of order $\ell$} if $x \mapsto {x}^{\bs} K(x)$ is integrable for all $|\bs| \leq \ell$ and 
    \begin{equation*}
        \int  K(x) \mathrm{d}x = 1 \mbox{ and } \int {x}^{\bs}  K(x) \mathrm{d}x = 0 \mbox{ for } |\bs| = 1, \ldots, \ell \, .
    \end{equation*}
    In particular, a kernel of order $\ell$ is orthogonal to any polynomial of degree $\leq \ell$ with no constant term.
\end{definition} 
\begin{assumption}[Assumptions on the kernel $K_d$]\label{ass:kernel}
The kernel $K_d$ satisfies the following:
\begin{itemize}
    \item[$(1a)$] $K_d$ is uniformly upper bounded by some $\kappa>0$,
    \item[$(1b)$] $K_d$ is of order $\beta$ (see Definition \ref{def:kernel_order}),

    \item[$(1c)$] The class of functions
    $$ \cF := \left\{ K_d \left( \frac{x-\cdot}{h}\right), x \in \dR^d, h>0 \right\}$$ satisfies 
    $ \sup_Q N(\cF,L^2(Q),\eps \| F \|_{L^2(Q)}) \leq \left( \frac{A}{\eps}\right)^v,$
\end{itemize}
where $A$ and $v$ are two positive numbers, $N(T,d,\eps)$ denotes the $\eps$-covering number \citep[see, e.g.][]{lafferty08} of the metric space $(T,d)$, $F$ is the envelope function of $\cF$ (i.e. $F(x):= \sup_{f \in \cF} |f(x)|$),  and the supremum is taken over the set of all probability measures on $\dR^d$. 
The quantity $v$ is called the $VC$ dimension of $\cF$.
\end{assumption}
Assumption $(1c)$ is a widely used condition, appearing in works such as \citet{gine2002rates, rinaldo2010generalized} and is fundamental to deriving the exponential inequality in \citet{liu2012exponential}. 
This assumption holds for a broad class of kernels, such as polynomial kernels with compact support and Gaussian kernels \citep{vandervaart1996, nolan87}.

\section{Von Mises 
Expansion for KL Divergence}\label{apx:vonmises}
We review a few formal definitions from the theory of influence functions here and derive our estimator for $\kl(\cdot\|\cdot)$ based on the Von Mises expansion.
The definitions are often given for functionals of a single distribution.
Since $\kl$ is a functional of two distributions, we require the extended definitions that apply to functionals of multiple distributions.
We follow \citep{kandasamy2015nonparametric} for this purpose.

Let $\mathcal{X}$ be a compact metric space.
Let $\mathcal{M}$ denote the set of all probability measures that are absolutely continuous with respect to Lebesgue\footnote{In general, the definitions can be adapted to an arbitrary measure $\mu$. We work with the Lebesgue measure for simplicity.}, and with Radon-Nikodym derivities lying in $L_2(\mathcal{X})$.
For $P,Q,H\in\mathcal{M}$ and a functional $T:\mathcal{M}\times\mathcal{M}\to\mathbb{R}$, 
the maps $T'_P,T'_Q:\mathcal{M}\times\mathcal{M}\to\mathbb{R}$ where 
\[T'_P(H;P,Q) = \frac{\partial T(P+tH,Q)}{\partial t}\Big\vert_{t=0},
\quad
T'_Q(H;P,Q) = \frac{\partial T(P,Q+tH)}{\partial t}\Big\vert_{t=0},\]
are called the Gâteaux derivatives at $(P,Q)$ if the derivatives exist and are linear in $H$.
We say $T(\cdot,\cdot)$ is Gâteaux differentiable at $(P,Q)$ if the Gâteaux derivatives exist at $(P,Q)$.
For a functional $T$ that is Gâteaux differentiable at $(P,Q)$, functions $\psi_P,\psi_Q:\mathcal{X}\to\mathbb{R}$ which satisfy the following equations are said to be the influence functions of $T$ with respect to $P$ and $Q$:
\[
\begin{split}
    T'_P(H_1-P;P,Q) = \int_\mathcal{X}\psi_P(x;P,Q)dH_1(x),
    \quad
    T'_Q(H_2-Q;P,Q) = \int_\mathcal{X}\psi_Q(x;P,Q)dH_2(x).
\end{split}
\]
It can be shown that the influence functions calculated below satisfy the equation above \citep{fernholz2012mises}:
\[
\begin{split}
    \psi_P(x;P,Q) = T'_P(\delta_x-P;P,Q) = \frac{\partial T\left((1-t)P+t\delta_x,Q\right)}{\partial t}\Big\vert_{t=0},\\
    \psi_Q(x;P,Q) = T'_Q(\delta_x-Q;P,Q) = \frac{\partial T\left(P,(1-t)Q+t\delta_x\right)}{\partial t}\Big\vert_{t=0}.
\end{split}
\]
Below, we derive these influence functions for $T\equiv\kl$.
We further restrict our attention to measures with continuous densities.
In what follows, $p,q$ denote the densities corresponding to $P,Q\in\mathcal{M}$.
\begin{equation*}
    \begin{aligned}
             \psi_p (x; p, q) &= \frac{\partial \kl \left((1-t) p + t \delta_x \| q \right) }{\partial t}\bigg|_{t=0} = \frac{\partial}{\partial t} \int_{\cX} \left( (1-t) p+t \delta_x \right) \log \left(\frac{(1-t) p+t \delta_x}{q}\right) \, d\tilde{x} \bigg|_{t=0}\\
             &=\int_{\cX} (- p + \delta_x) \log(\frac{p}{q}) +\int_{\cX}(- p+ \delta_x) \, d\tilde{x}\\
             &= -\kl(p \| q)+\log\frac{p(x)}{q(x)}.
    \end{aligned}
\end{equation*}
Similarly,
\begin{equation*}
\begin{aligned}
         \psi_q (x; p, q)& = \frac{\partial \kl \left( p \| (1-t) q + t \delta_x   \right) }{\partial t}\bigg|_{t=0}=  \frac{\partial}{\partial t} \int_{\cX} p \log \left( \frac{p}{(1-t)q+t \delta_x} \right) \, d\tilde{x} \bigg|_{t=0}\\
         &= \int_{\cX} p - \frac{p}{q} \delta_x \, d\tilde{x} = 1- \frac{p(x)}{q(x)}.
\end{aligned}
\end{equation*}
Given these influence functions and approximations $\hat{p}$ and $\hat{q}$ of $p$ and $q$, the first-order Von Mises expansion can be written as
\begin{equation}\label{eq:expansion}
    \begin{aligned}
            \kl (p \| q) = & \kl(\hat{p} \| \hat{q}) + \int \psi_{\hat{p}} (x; \hat{p}, \hat{q}) p(x) \, dx + \int  \psi_{\hat{q}} (x; \hat{p}, \hat{q}) q(x) \, dx \\
            &+ \mathcal{O} (\| p-\hat{p} \|_2^2) + \mathcal{O} (\| q-\hat{q} \|_2^2)\\
            =& \kl(\hat{p} \| \hat{q}) + \int \left( -\kl(\hat{p} \| \hat{q})+log(\frac{\hat{p}}{\hat{q}} )\right) p(x) \, dx + \int \left( 1- \frac{\hat{p}}{\hat{q}} \right ) q(x) \, dx \\
             &+ \mathcal{O} (\| p-\hat{p} \|_2^2) + \mathcal{O} (\| q-\hat{q} \|_2^2)\\
             =& \int p(x) \log \frac{\hat{p}(x)}{\hat{q}(x)} \, dx - \int q(x) \frac{\hat{p}(x)}{\hat{q}(x)} \, dx + 1\\
              &+ \mathcal{O} (\| p-\hat{p} \|_2^2) + \mathcal{O} (\| q-\hat{q} \|_2^2).
    \end{aligned}
\end{equation}
It is clear from this expansion that the difference between $\kl(p\|q)$ and 
\begin{equation}\label{eq:klvmest}
    \int p(x) \log \frac{\hat{p}(x)}{\hat{q}(x)} \, dx - \int q(x) \frac{\hat{p}(x)}{\hat{q}(x)} \, dx + 1
\end{equation}
is bounded by second-order remainder terms.
To construct an estimator based on Von Mises expansion, we use data split.
In particular, we estimate $\hat{p},\hat{q}$ using one half of the data, while using the other half to compute the sample analogue of Eq.~\eqref{eq:klvmest}.
Specifically, given iid samples $\{x_1^i\}_{i=1}^m$ and $\{x_2^j\}_{j=1}^n$ drawn from $p$ and $q$ respectively, our data-split estimator based on Von Mises expansion is
\[
     \klv =\frac{2}{n} \sum_{i=n/2+1}^n \log \frac{\hat{p}(x_1^i)}{\hat{q}(x_1^i)} - \frac{2}{m} \sum_{j=m/2+1}^m \frac{\hat{p}(x_2^j)}{\hat{q}(x_2^j)} +1,
\]
where samples $\{x_1^i\}_{i=1}^{m/2}$ and $\{x_2^j\}_{j=1}^{n/2}$ are used to estimate $\hat{p}$ and $\hat{q}$, respectively.

\section{Omitted Proofs}\label{apx:proofs}
\thmexpkl*
\begin{proof}
In order to avoid the notational burden, we drop the subscripts and denote the kernel density estimators simply by $\hat{p}$ and $\hat{q}$ throughout this proof.
From Eq.~\eqref{eq:expansion} and by definition of $\klv$,
 \begin{equation}\label{eq:4terms}
 \begin{split}
          \klv - \kl(p \| q) = & \Big(\frac{2}{n} \sum_{i=n/2+1}^n \log \frac{\hat{p}(x^i)}{\hat{q}(x^i)}  -  \int p(x) \log \frac{\hat{p}(x)}{\hat{q}(x)} \, dx\Big)\\ &- \Big(\frac{2}{m} \sum_{j=m/2+1}^m \frac{\hat{p}(y^j)}{\hat{q}(y^j)}- \int q(x) \frac{\hat{p}(x)}{\hat{q}(x)} \, dx\Big)
              \\ &+ \mathcal{O} (\| p-\hat{p} \|_2^2) + \mathcal{O} (\| q-\hat{q} \|_2^2).
 \end{split}
 \end{equation}
We will bound each term on the right-hand side separately.
First note that from Proposition~\ref{prp:conc_sup}, $\hat{p}$ and $\hat{q}$ uniformly converge to $p$ and $q$, respectively.
Since $p(x)>p_{\min}$ and $q(x)>q_{\min}$ on the (compact) set $\cX$, for large enough $n$ and $m$ (represented by constant thresholds $n_0$ and $m_0$, respectively), we have that 
\[\hat{p}(x)>p_{\min}/2,\quad\hat{q}(x)>q_{\min}/2,\]
almost surely.
Thus, every term in the sum $ \sum_{i=n/2+1}^n \frac{2}{n} \log \frac{\hat{p}(x^i)}{\hat{q}(x^i)}$ is almost surely bounded by $c/n$ where $c$ is a positive constant.
Azuma-Hoeffding inequality implies 
\begin{equation}\label{eq:term1}
    \begin{aligned}
            Pr \left( \left | \frac{2}{n} \sum_{i=n/2+1}^n \log \frac{\hat{p}(x^i)}{\hat{q}(x^i)} - \int p(x) \log \frac{\hat{p}(x)}{\hat{q}(x)} \, dx  \right | > \epsilon/4 \right )  &\leq 2 \exp \left( \frac{-(\frac{\epsilon}{4})^2}{2 \sum_{i=n/2+1}^n (\frac{c}{n})^2 } \right)\\
            & \leq C_1 \exp \left( -C'_1 n^{1/2} \epsilon \right),
    \end{aligned}
\end{equation}
where the last inequality holds since $n^{1/2} \epsilon>1$ by assumption.

With the same reasoning, we can conclude that each term in the sum $ \sum_{j=m/2+1}^m \frac{2}{m} \frac{\hat{p}(y^j)}{\hat{q}(y^j)}$  is almost surely bounded by $c'/m$ where $c'$ is a positive constant.
Again, by Azuma-Hoeffding inequality we get
 \begin{equation}\label{eq:term2}
     \begin{aligned}
         Pr \left( \left| \frac{2}{m} \sum_{j=m/2+1}^m \frac{\hat{p}(y^j)}{\hat{q}(y^j)} - \int q(x) \frac{\hat{p}(x)}{\hat{q}(x)} \, dx \right | > \epsilon/4 \right) & \leq 2 \exp \left( \frac{-(\frac{\epsilon}{4})^2}{2 \sum_{j=m/2+1}^m (\frac{c'}{m})^2 }  \right)\\
         & \leq C_2 \exp \left( -C'_2 m^{1/2} \epsilon \right) 
     \end{aligned}
 \end{equation}
where the last inequality holds since $m^{1/2} \epsilon>1$ by assumption.

For the third term on the right-hand side of Eq.~\eqref{eq:4terms}, 

\begin{equation}\label{eq:term3}
    \begin{aligned}
        Pr(\|p-\hat{p}\|_2^2>\epsilon/4)&=
         Pr \left( \int_\mathcal{X} \left( p(x) - \hat{p}(x) \right)^2 dx > \epsilon/4 \right) \\&\leq Pr\left( \sup_{x \in \cX} |p(x) - \hat{p}(x)| > \frac{\sqrt{\epsilon/4}}{ \text{Vol}(\cX)} \right)\leq C_3 \exp \left( -C'_3 n^{\frac{2\beta_p}{2 \beta_p +d}} \epsilon \right),
    \end{aligned}
\end{equation}
where the last inequality follows from Proposition~\ref{prp:conc_sup}.
In a similar way, we get:
\begin{equation}\label{eq:term4}
         Pr(\|q-\hat{q}\|_2^2>\epsilon/4) = Pr \left( \int_\mathcal{X} \left( q(x) - \hat{q}(x) \right)^2 dx > \epsilon/4 \right) \leq C_4 \exp \left( -C'_4 m^{\frac{2\beta_q}{2 \beta_q +d}} \epsilon \right).
\end{equation}
Finally, combining Equations \eqref{eq:term1}, \eqref{eq:term2}, \eqref{eq:term3}, and \eqref{eq:term4} with a union bound, we arrive at the desired inequality:
\begin{equation*}
    \begin{aligned}
         Pr \Big( \left|  \klv(\hat{p}_{h_p}\|\hat{q}_{h_q}) - \kl(p \| q)\right | &> \epsilon \Big) \leq   C_1 \exp \Big( -C'_1 n^{1/2} \epsilon \Big) + C_2 \exp \Big( -C'_2 m^{1/2} \epsilon \Big)\\
         & + C_3 \exp \Big( -C'_3 n^{\frac{2\beta_p}{2 \beta_p +d}} \epsilon \Big) + C_4 \exp \Big( -C'_4 m^{\frac{2\beta_q}{2 \beta_q +d}} \epsilon \Big).
    \end{aligned}
\end{equation*}
\end{proof}

\lemkl*
\begin{proof}

By Assumption \ref{as:cor} we have:
    \begin{equation*}
            \kl \left(P(A,B) \|P(A)P(B) \right) > \epsilon = \int \int P(a, b) \log \frac{P(a,b)}{P(a)P(b)} \, da db > \epsilon.
    \end{equation*}
Therefore,
\begin{equation*}
    \begin{aligned}
            & \int \left ( \int P(b|a) \log \frac{P(b|a)}{P(b)} \, db \right) P(a) da > \epsilon\\
            &  \Longrightarrow \quad \mathbb E_{B} \left[ \kl \left( P(b|a) \| P(b) \right)  \right] > \epsilon.
    \end{aligned}
\end{equation*}

Similarly, we can conclude that:

\begin{equation*}
    \mathbb E_{A} \left[ \kl \left( P(a|b) \| P(a) \right)  \right] > \epsilon.
\end{equation*}
\end{proof}

\lemconstant*

\begin{proof}
Define event $\mathcal{E}_i=  \mathbbm{1}\{ h(a_i) > 2^{-j^*} \}$, $ \mathcal{E}= \sum_{i=1}^{l_j^*} \mathcal{E}_i$, and $\mu := \mathbb E [ \mathcal{E}]$.
From Lemma \ref{lemma: levin}, $\mathbb E[\mathcal{E}_i] > r_j^*$ for every $i \in [l_{j^*}]$, and therefore,
\begin{equation}\label{eq: mu>lr}
    \mu > l_{j^*}r_{j^*}.
\end{equation}

Additionally, let $\gamma := 2c+1- \sqrt{2c(2c+2)}$.
Note that since $(2c+1)^2> 2c(2c+2)$ therefore $\gamma >0$.

It can also be verified that:

\begin{equation}\label{eq: gammac}
   (1-\frac{1+\gamma}{2c+2})^2 = \frac{c}{c+1} .
\end{equation}

Now, set $\alpha = 1- \frac{1+\gamma}{\mu}$. 
Observe that $0<\alpha<1$, since $\mu > l_{j^*} r_{j^*} = 2c+2 > \gamma+1$.

Next, we bound $Pr(\mathcal{E}<1)$ as follows:

\begin{equation*}
    \begin{aligned}
         Pr(\mathcal{E} < 1) &\leq Pr \left( \mathcal{E} \leq 1+ \gamma \right) = Pr \left( \mathcal{E} \leq (1-\alpha) \mu \right)\overset{(a)} {\leq} \exp(-\frac{\mu \alpha^2}{2})\\
         &= \exp \left(-\frac{\mu \left(1- \frac{1+\gamma}{\mu} \right)^2}{2} \right) \overset{(b)}{\leq} \exp(-\frac{l_{j^*} r_{j^*} (1- \frac{1+\gamma}{l_{j^*} r_{j^*}})^2}{2})\\
         & = \exp \left(-\frac{(2c+2) \left(1- \frac{1+\gamma}{2c+2} \right)^2}{2} \right)\\
         &\overset{(c)}{=} \exp \left(-\frac{ (2c+2) \frac{c}{c+1} }{2} \right) = e^{-c}
    \end{aligned}
\end{equation*}
 Here, $(a)$ follows from the Chernoff bound, $(b)$ from \eqref{eq: mu>lr}, and $(c)$ from \eqref{eq: gammac}.

Finally, we conclude that:
\begin{equation*}
    Pr \left(\exists \quad \text{sample $a_i$ that satisfies} \quad h(a_i)> r_{j^*} \right) = 1-Pr(\mathcal{E}<1) \geq 1-e^{-c}.
\end{equation*}
\end{proof}

\leminterv*
\begin{proof}
     \begin{equation*}
        \begin{aligned}
            \mathbb{E}_{(A,\tilde{A})\sim P(A)\times P(A)} \Big[\kl \Big(P (B|A) \| &P(B|\tilde{A}) \Big) \Big] \\& \overset{(a)}{\geq}
            \mathbb{E}_{A\sim P(A)} \left[\kl \left(P (B|A) \| \mathbb{E}_{\tilde{A} \sim P(A)} 
            \left [P(B|\tilde{A}) \right ] \right) \right] \\
            & = \mathbb{E}_{A\sim P(A)} \left[\kl \left(P (B|A) \| \sum_{\tilde{A}}P(\tilde{A}) P(B|\tilde{A}) \right) \right] \\
            & = \mathbb{E}_{A\sim P(A)} \left[\kl \left(P (B|A) \| P(B)  \right) \right] \overset{(b)}{>}\epsilon.
        \end{aligned} 
     \end{equation*}
Here, $(a)$ follows from Jensen's inequality and $(b)$ holds by Lemma \ref{lem:expkl}.

The proof of the second inequality is identical.

\end{proof}
\end{document}